\documentclass[letterpaper]{article} 
\usepackage{aaai2026}  
\usepackage{times}  
\usepackage{helvet}  
\usepackage{courier}  
\usepackage[hyphens]{url}  
\usepackage{graphicx} 
\usepackage{natbib}  
\usepackage{caption} 
\usepackage{subfig}
\usepackage[hyphens]{url}  
\usepackage{booktabs}  
\usepackage{amsthm}
\newtheorem{theorem}{Theorem}

\newtheorem{assumption}{Assumption}

\usepackage[table]{xcolor}  
\usepackage{tabularx} 
\usepackage{algorithm}
\usepackage{algorithmic}

\usepackage{amsmath}
\usepackage{amssymb}
\usepackage{multirow}
\usepackage{mathtools} 

\usepackage{newfloat}
\usepackage{listings}
\DeclareCaptionStyle{ruled}{labelfont=normalfont,labelsep=colon,strut=off} 
\floatstyle{ruled}
\newfloat{listing}{tb}{lst}{}
\floatname{listing}{Listing}
\title{Seeing through the Conflict: Transparent Knowledge Conflict Handling in Retrieval-Augmented Generation}

\author {
    Hua Ye\equalcontrib\textsuperscript{\rm 1},
    Siyuan Chen\equalcontrib\textsuperscript{\rm 2},
    Ziqi Zhong\textsuperscript{\rm 3},
    Canran Xiao\textsuperscript{\rm 4}\thanks{Corresponding author.},
    Haoliang Zhang\textsuperscript{\rm 5},
    Yuhan Wu\textsuperscript{\rm 6},
    Fei Shen\textsuperscript{\rm 7}
}
\affiliations {
    \textsuperscript{\rm 1}School of Management \& Engineering, Nanjing University, Nanjing, China\\
    \textsuperscript{\rm 2}School of Computer Science, University of Bristol, Bristol, United Kingdom\\
    \textsuperscript{\rm 3}Department of Management, London School of Economics and Political Science, London, United Kingdom\\
    \textsuperscript{\rm 4}School of Cyber Science and Technology, Shenzhen Campus of Sun Yat-sen University, Shenzhen, China\\
    \textsuperscript{\rm 5}Electrical and Computer Engineering, The University of Oklahoma, Oklahoma, USA\\
    \textsuperscript{\rm 6}College of Computer Science and Technology, Zhejiang University, Hangzhou, China\\
    \textsuperscript{\rm 7}NExT++ Research Centre, National University of Singapore, Singapore\\
    522024150103@smail.nju.edu.cn, gd23774@bristol.ac.uk, z.zhong6@lse.ac.uk, xiaocanran999@gmail.com, mars\_zhang@ou.edu, wuyuhan@zju.edu.cn, shenfei29@nus.edu.sg
}

\usepackage{bibentry}

\begin{document}

\maketitle

\begin{abstract}
Large language models (LLMs) equipped with retrieval---the
Retrieval-Augmented Generation (RAG) paradigm---should combine
their parametric knowledge with external evidence, yet in practice they
often hallucinate, over-trust noisy snippets, or ignore vital context.
We introduce \textsc{TCR} (\emph{Transparent Conflict Resolution}), a
plug-and-play framework that makes this decision process observable and
controllable.  TCR \textbf{(i)} disentangles \textit{semantic match} and
\textit{factual consistency} via dual contrastive encoders,  
\textbf{(ii)} estimates \textit{self-answerability} to gauge confidence
in internal memory, and  
\textbf{(iii)} feeds the three scalar signals to the generator through a
lightweight soft-prompt with SNR-based weighting.  
Across seven benchmarks TCR improves conflict detection
(+5--18\,F\textsubscript{1}), raises knowledge-gap recovery by
$\mathbf{+21.4}$\,pp and cuts misleading-context overrides by
$\mathbf{-29.3}$\,pp, while adding only 0.3 \% parameters.  
The signals align with human judgements and expose temporal decision patterns.
\end{abstract}


\section{Introduction}

Recent advancements in large language models (LLMs) have significantly improved the performance and applicability of natural language processing (NLP) systems across various tasks, including question answering, text summarization, and conversational AI~\citep{diao-etal-2024-learning,Diao_2025_WACV,zhang2025enhancing,wang2025medical}. Despite their remarkable generative capabilities, LLMs often struggle with accurately recalling specific factual information~\citep{li2025modeling,huang2024gaussianmarker}, leading to frequent hallucinations or inconsistent responses, particularly in knowledge-intensive domains~\citep{yao2023ndc,tong2025does,xiao2025diffusion}.

\begin{figure}[H]
	\centering
	\includegraphics[width=1\linewidth]{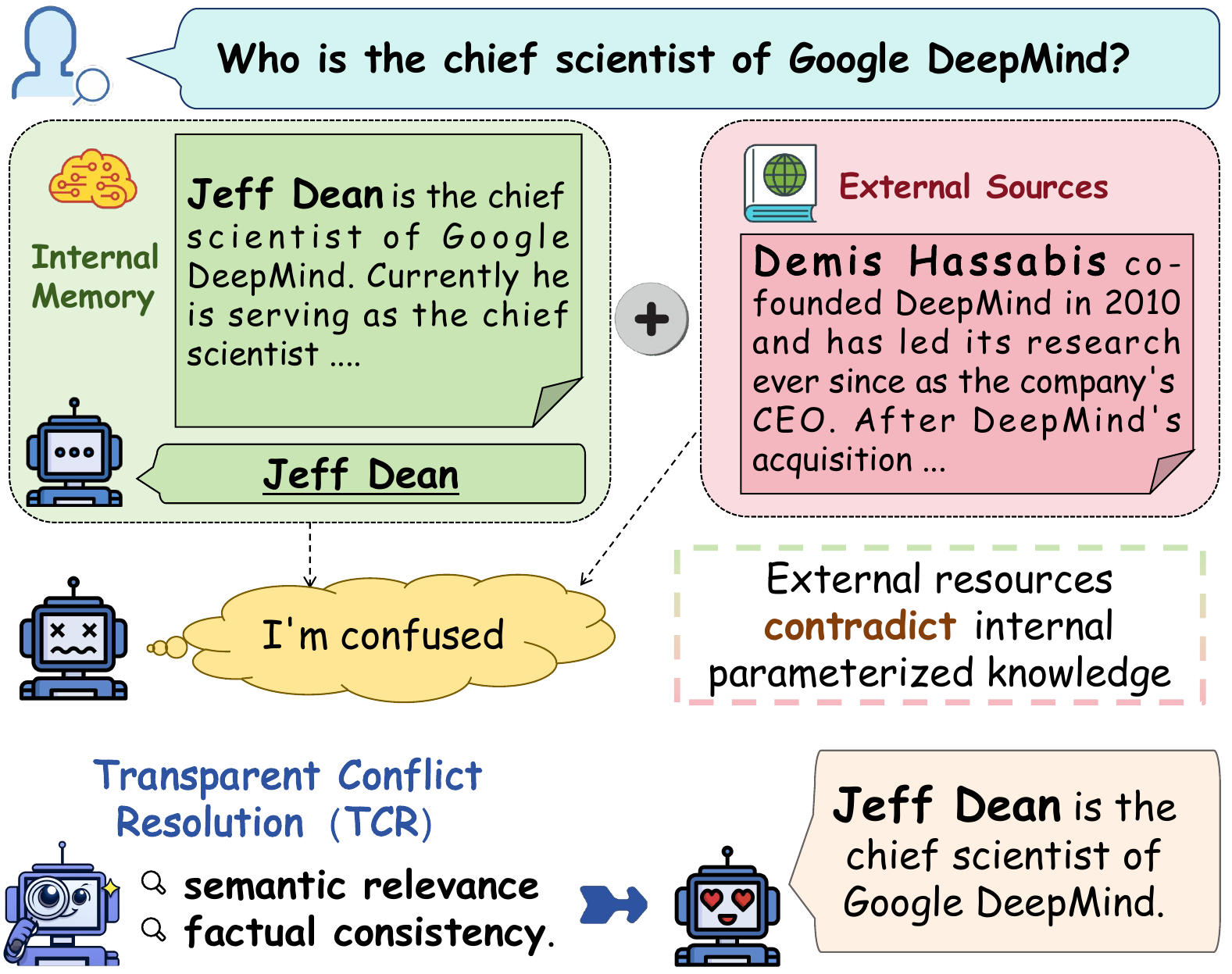}
	\caption{\textbf{Knowledge conflicts in a typical RAG system}. Our method disentangles semantic relevance from factual consistency, detects conflicts, and injects lightweight prompt signals to steer RAG models toward faithful, knowledge-aligned generation.}
	\label{fig:teaser}
\end{figure}

Retrieval-Augmented Generation (RAG) enhances LLMs by combining internal parametric knowledge with external non-parametric knowledge retrieved at inference~\citep{lewis2020retrieval}. Explicitly integrating external knowledge significantly improves factual accuracy and contextual relevance for tasks such as question answering and fact verification~\citep{chen-etal-2022-rich, yu-etal-2023-characterizing}. However, conflicts between implicitly stored internal knowledge and retrieved external information remain a critical issue, often causing inconsistent or incorrect outputs~\citep{longpre-etal-2021-entity, xie2024adaptive,tao2023dudb}.

Previous studies highlight variability in how LLMs handle knowledge conflicts, with models either prioritizing internal parametric knowledge~\citep{longpre-etal-2021-entity} or external retrieved information~\citep{chen-etal-2022-rich}. Recent research reveals models internally encode signals indicating knowledge discrepancies and answerability, yet fail to utilize them effectively due to sequential forward propagation constraints. Current resolution methods either explicitly modify outputs, demanding extensive dataset-specific tuning and limiting generalization~\citep{shi2025ircan, jin-etal-2024-tug,wang2024computing,jiang2025transforming}, or rely on black-box APIs and manual interventions, reducing transparency and autonomy~\citep{wang2024resolving}. Fig.~\ref{fig:teaser} illustrates a typical knowledge conflict scenario in a RAG system, where internal parametric knowledge and external retrieved information clash, resulting in uncertainty.

To address these limitations, we propose a novel framework that integrates transparent conflict detection and autonomous resolution into RAG. Our method:  
1) Detects knowledge conflicts between internal model knowledge and external retrieved information,  
2) Extracts meaningful signals from model representations to guide generation toward accurate outputs,  
3) Ensures interpretability and compatibility with existing RAG architectures without manual intervention, improving scalability and robustness.   Our key contributions are summarized as follows:

1)~A transparent semantic vector-based conflict detection method, enhancing interpretability over black-box approaches~\citep{wang2024resolving}.  

2)~An autonomous conflict-resolution mechanism using soft prompt tuning, enabling accurate, consistent generation without manual intervention. Plug-and-play for easy integration.  

3)~Extensive experiments show our method reduces factual errors, advancing RAG reliability in knowledge-intensive tasks.  


\section{Related Work}
\paragraph{Retrieval-Augmented Generation.}

Large language models (LLMs) effectively encode factual knowledge parametrically, yet suffer from hallucinations~\citep{chen2024benchmarking}, rare entity retention issues~\citep{kandpal2023large}, and temporal degradation~\citep{jang-etal-2022-temporalwiki}. Retrieval-Augmented Generation (RAG) mitigates these by combining parametric knowledge with external retrieval~\citep{lewis2020retrieval}. Early approaches pre-trained smaller retrieval-augmented models~\citep{guu2020retrieval}, while recent work exploits large models' in-context learning~\citep{yu2023generate, ram-etal-2023-context}. However, resolving conflicts between retrieved and parametric knowledge remains challenging~\citep{min2023augmented, ni-etal-2024-llms,zhong_2025}.

\paragraph{Context-Memory Conflicts in LLMs.}
Context-memory conflict occurs when retrieved context contradicts internal knowledge~\citep{longpre-etal-2021-entity, chen-etal-2022-rich,diao-etal-2025-temporal}, often due to temporal mismatches, misinformation in documents~\citep{pan-etal-2023-attacking}, or misleading prompts~\citep{xu-etal-2024-earth}. Models exhibit inconsistent conflict resolution, sometimes favoring parametric knowledge~\citep{longpre-etal-2021-entity} and other times prioritizing coherent context~\citep{chen-etal-2022-rich, su2024conflictbank}. Additional challenges include confirmation bias~\citep{xie2024adaptive} and susceptibility to deceptive inputs~\citep{ying-etal-2024-intuitive}. Recent studies emphasize models' limited robustness in such scenarios, highlighting the need for better solutions~\citep{su2024conflictbank,xiao2024confusion}.

\paragraph{Resolving Knowledge Conflicts.} 
Existing conflict-resolution strategies fall into five categories: (1) \textit{Context-faithful} methods prioritize external knowledge via fine-tuning~\citep{li-etal-2023-large, gekhman-etal-2023-trueteacher, xue-etal-2023-improving}, prompting~\citep{zhou-etal-2023-context}, decoding techniques, or knowledge plugins; (2) \textit{Memory-faithful} approaches counter misinformation via vigilance mechanisms~\citep{xu-etal-2024-earth}, with some advocating parametric knowledge in noisy contexts~\citep{zhang2024raft}; (3) \textit{Source-distinguishing} methods present multiple source-derived answers for user selection~\citep{neeman-etal-2023-disentqa}; (4) \textit{Factuality-enhancing} techniques improve correctness via hybrid knowledge fusion~\citep{zhang-etal-2023-merging} or contrastive decoding; and (5) \textit{Structural modifications} adjust architectures to optimize knowledge integration~\citep{shi2025ircan, jin-etal-2024-cutting}. 

Current methods remain biased toward either context or memory, lacking adaptive balancing mechanisms~\citep{wang2024resolving}. Our work bridges this gap with an interpretable, dynamic conflict-resolution framework.

\section{Exploratory Analysis}
\label{sec:exploratory}
We first investigate whether sentence embeddings disentangle different information dimensions (e.g., meaning vs. truthfulness) and assess if off-the-shelf encoders inherently capture signals for conflict detection. We address:  
\textbf{(Q1)} Can standard encoder embedding spaces separate \emph{what a sentence is about} (meaning) from \emph{whether its fact is true} (truthfulness)?  
\textbf{(Q2)} If such signals exist, how reliably do they expose knowledge conflicts in RAG systems?  

\vspace{1mm}
\noindent\textbf{Experimental Setup.} We sample $5\,\text{k}$ factual triples from Wikidata
and automatically create, for each fact, three surface forms:
(i)~\textbf{Paraphrase} (same fact, different wording),  
(ii)~\textbf{Contradiction} (negated relation), and  
(iii)~\textbf{Unrelated} (different subject).  
All sentences are embedded using two public models of contrasting
capacity: \texttt{all-MiniLM-L6-v2} (384-d, 6 layers) and
\texttt{E5-large-v2} (1024-d, 24 layers).
Fig.~\ref{fig:semantic-factual} visualises the raw embeddings after
projection onto their first three principal components.

\begin{figure}[ht]
  \centering
  \includegraphics[width=0.45\textwidth]{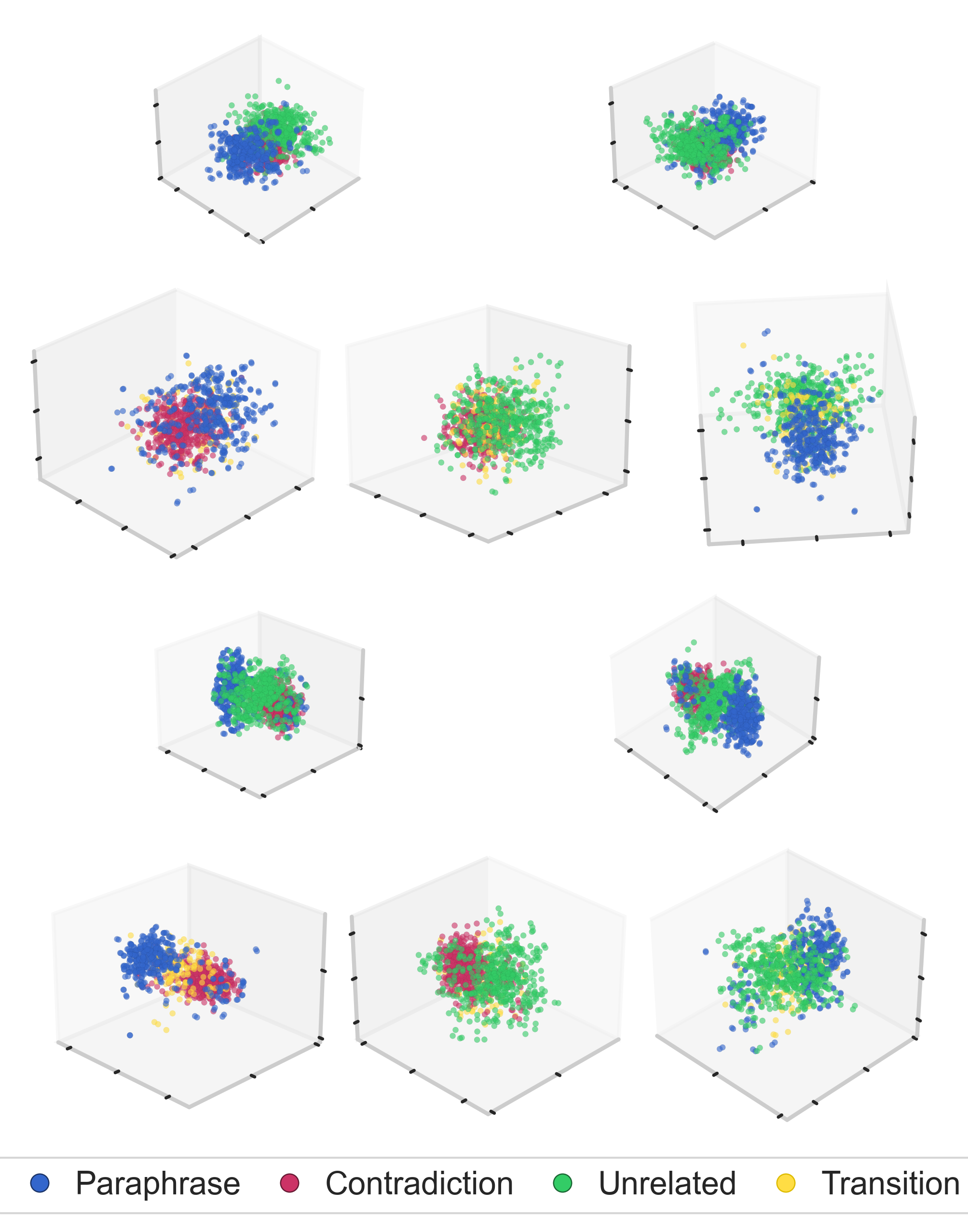}
  \caption{\textbf{Sentence‐embedding clusters for \textsc{all-MiniLM-L6-v2} (top)
vs.\ \textsc{E5-large-v2} (bottom)}.  
Rows 1 \& 3: two PCA views of 3-D space—paraphrase (blue), contradiction
(red), unrelated (green).  
Rows 2 \& 4: boundary regions (yellow) highlight transition zones.
MiniLM (384 d) forms a single blob; the larger E5 (1024 d, 24 layers) begins to separate truth from meaning.}
\label{fig:semantic-factual}
\end{figure}

\vspace{1mm}
\noindent\textbf{Global Embedding Structure.} In Fig.~\ref{fig:semantic-factual} we observe that the lighter
MiniLM encoder yields one dominant cluster where paraphrases,
contradictions, and unrelated sentences overlap.  
The larger E5 encoder begins to form two visible sub-manifolds, yet the
transition regions (yellow points) remain fuzzy—suggesting that
\emph{model scale alone does not guarantee explicit separation of
meaning and truth}.  This lack of structure motivates a disciplined
decoupling strategy.

\begin{figure}[t]
  \centering
  \includegraphics[width=0.47\textwidth]{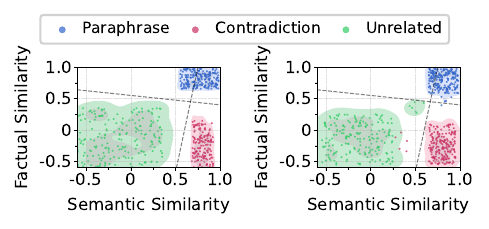}
  \caption{\textbf{2D scatter plot of semantic-factual similarity space}. Left: Sentence pairs show clear separation of paraphrases (top right), contradictions (bottom right), and unrelated sentences (bottom left) using just one linear projection layer. Contour lines show KDE boundaries. Right: Example of annotated transitions (other types perform similarly well). 
  }
\label{fig:semantic-factual-space}
\end{figure}

\vspace{1mm}
\noindent\textbf{Identifying a Semantic–Truth Subspace.} 
We next train a single linear probe to project each embedding onto a
2-D plane that maximises class separability while keeping
paraphrase pairs close.  The resulting \emph{semantic–truth space} is
shown in Fig.~\ref{fig:semantic-factual-space}.  A distinctive
triangular pattern appears:
paraphrases cluster in the upper-right (high meaning, high truth),
contradictions in the lower-right (high meaning, low truth), and
unrelated sentences in the lower-left (low meaning, low truth).
Kernel-density contours highlight crisp boundaries, and manual
inspection of boundary points (right panel) confirms smooth, interpretable
gradients rather than abrupt jumps.

\vspace{1mm}
\noindent\textbf{Findings}.
Our probe reveals that off-the-shelf embeddings intertwine meaning and truth, yet a single linear projection cleanly separates them and yields smooth, continuous scores. This motivates (i)~dedicated semantic and factual encoders, (ii)~independent contrastive losses to sharpen each axis, and (iii)~soft scalar conflict signals fed into generation—choices.

\section{Proposed Method}

\begin{figure*}[htb]
	\centering
	\includegraphics[width=0.9\linewidth]{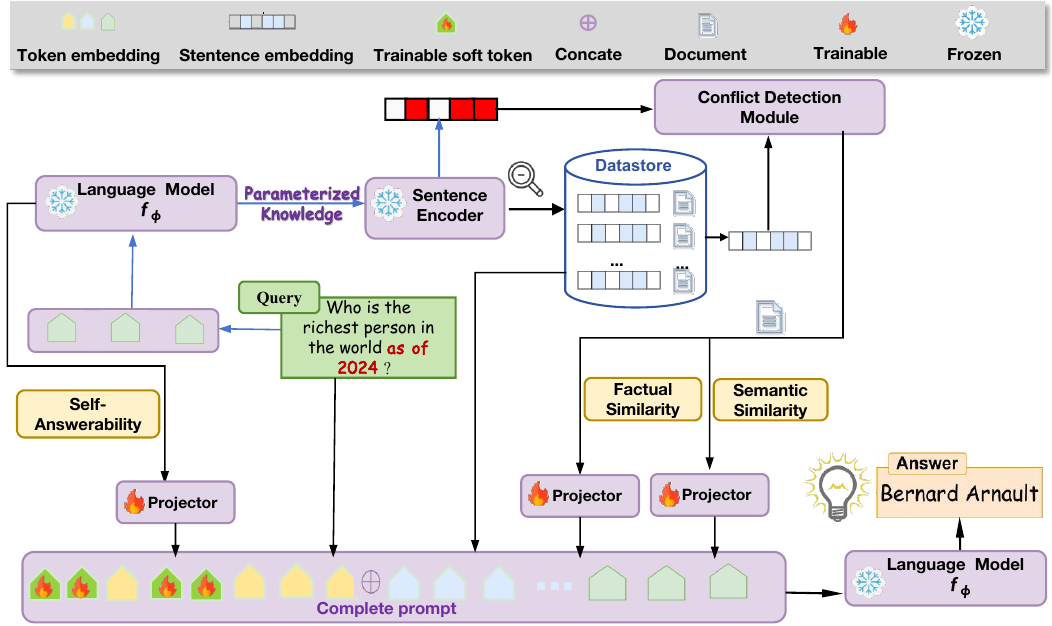}
	\caption{\textbf{Overview of our \underline{T}ransparent \underline{C}onflict \underline{R}esolution (TCR) framework for Retrieval-Augmented Generation (RAG)}. Given a query, TCR identifies knowledge conflicts through semantic and factual similarity vectors, generating explicit conflict signals. These signals, along with self-answerability estimation, are integrated using soft prompt tuning to guide the Language Model's generation toward enhanced factual consistency and interpretability.}
    \vspace{-5mm}
	\label{fig:pipeline}
\end{figure*}

We propose a novel framework to resolve knowledge conflicts in RAG. Given query \(q\), RAG generates response \(y\) conditioned on external context \(c\) and internal parametric knowledge \(\theta\):  
$
    p(y|q,c;\theta).  
$ 
Conflicts arise when \(c\) contradicts \(\theta\), leading to inconsistent outputs. Our method detects such conflicts and guides generation via conflict-aware signals. The framework includes: (1) a contrastive learning-based conflict detection module to decouple semantics from factual accuracy, and (2) a conflict-aware generation module using soft prompt tuning. The pipeline is shown in Fig.~\ref{fig:pipeline}.  

\subsection{Conflict Detection via Contrastive Learning}

\paragraph{Dataset Construction.} 
We construct a conflict-focused dataset \(\mathcal{D}\), derived from knowledge triples in Wikidata and expanded via GPT-4o. Each original statement \(s\) is associated with 3 types of statements: paraphrases (\(s^{+}_{para}\)) preserving semantic and factual equivalence, irrelevant statements (\(s^{-}_{irr}\)) with no semantic relation, and conflicting statements (\(s^{-}_{conf}\)) that maintain semantic coherence but introduce factual contradictions. This structured diversity simulates scenarios of knowledge conflicts faced by RAG systems.

\paragraph{Semantic-Factual Feature Decoupling.}
Effective conflict detection requires separating semantic similarity from factual correctness. We employ two encoders—semantic encoder \(E_{sem}\) and factual encoder \(E_{fact}\)—based on the SFR retrieval model. Each textual statement \(s\) is encoded into two distinct vector spaces as follows:
\begin{equation}
    \mathbf{z}_{sem} = E_{sem}(s), \quad \mathbf{z}_{fact} = E_{fact}(s),
\end{equation}
where semantic vectors \(\mathbf{z}_{sem}\) capture meaning and topical coherence, and factual vectors \(\mathbf{z}_{fact}\) encode factual validity.

\paragraph{Contrastive Learning Objective.}
We train the semantic and factual encoders using separate contrastive losses. Specifically, for each original statement \( s \):

The semantic contrastive loss (\(\mathcal{L}_{sem}\)) is defined as:
\begin{equation}
\begin{split}
&\mathcal{L}_{sem}(s) = \\
-&\log \frac{\sum_{s' \in \{s^{+}_{para}, s^{-}_{conf}\}} 
e^{\text{sim}(\mathbf{z}_{sem}(s),\mathbf{z}_{sem}(s'))/\tau}}{
\sum_{s' \in \{s^{+}_{para}, s^{-}_{conf}, s^{-}_{irr}\}} 
e^{\text{sim}(\mathbf{z}_{sem}(s),\mathbf{z}_{sem}(s'))/\tau}}
\end{split}
\end{equation}

The factual contrastive loss (\(\mathcal{L}_{fact}\)) is similarly defined as:
\begin{equation}
\begin{split}
&\mathcal{L}_{fact}(s) = \\
-&\log \frac{\sum_{s' \in \{s^{+}_{para}\}} 
e^{\text{sim}(\mathbf{z}_{fact}(s),\mathbf{z}_{fact}(s'))/\tau}}{
\sum_{s' \in \{s^{+}_{para}, s^{-}_{conf}, s^{-}_{irr}\}} 
e^{\text{sim}(\mathbf{z}_{fact}(s),\mathbf{z}_{fact}(s'))/\tau}}
\end{split}
\end{equation}

The overall contrastive training objective combines these two losses over all samples \( s \) in dataset \(\mathcal{D}\):
\begin{equation}
\mathcal{L}_{ctr} = \sum_{s \in \mathcal{D}}\left(\mathcal{L}_{sem}(s) + \mathcal{L}_{fact}(s)\right).
\end{equation}

\subsection{Conflict-Aware Generation via Soft Prompt Tuning}

\paragraph{Extraction and Integration of Conflict Signals.}
The conflict detection module explicitly provides scalar signals representing semantic similarity (\(\sigma_{sem}\)), factual consistency (\(\sigma_{fact}\)), and model-inferred answerability (\(\sigma_{ans}\)). These signals are formally extracted as:
\begin{align}
\sigma_{sem}(q, c) &= \text{sim}(\mathbf{z}_{sem}(q),\mathbf{z}_{sem}(c)), \\
\sigma_{fact}(q, c) &= \text{sim}(\mathbf{z}_{fact}(q),\mathbf{z}_{fact}(c)), \\
\sigma_{ans}(q) &= A(q;\theta),
\end{align}
where \(A(\cdot)\) is the model’s answerability estimation following \citet{slobodkin-etal-2023-curious}. These signals are then projected into embeddings via dedicated MLP projectors:
\begin{equation}
    \mathbf{e}_{signal} = \text{MLP}\left([\sigma_{sem},\sigma_{fact},\sigma_{ans}]\right).
\end{equation}

\paragraph{Soft Token Embedding Strategy.}
To guide generation with conflict-aware cues, we prepend \emph{soft tokens}—trainable embeddings seeded from the base model—with the conflict signals, yielding an augmented embedding sequence that directs the model’s attention
$
\mathbf{x}' = [\mathbf{e}_{soft},\mathbf{e}_{signal},\mathbf{x}],
$
where \(\mathbf{x}\) represents the original input embeddings. This approach allows the model to leverage conflict-awareness explicitly during response generation, improving factual consistency.

\paragraph{Dynamic Loss Weighting via Signal-to-Noise Ratio.}
To optimize the integration of multiple conflict signals, we dynamically weight their training contributions according to the Signal-to-Noise Ratio (SNR). The SNR for each signal \(i \in \{\text{sem}, \text{fact}, \text{ans}\}\) is computed as:
\begin{equation}
    \text{SNR}_i = \frac{\text{Var}(\hat{y}_i)}{\text{Var}(y-\hat{y}_i)},
\end{equation}
where \(\hat{y}_i\) is the prediction made solely using signal \(i\), and \(y\) is the ground-truth label indicating presence or absence of a conflict. Signals with higher SNR values contribute more significantly, weighted by:
\begin{equation}
w_i = \frac{\exp(\text{SNR}_i)}{\sum_j \exp(\text{SNR}_j)}.
\end{equation}

The combined loss for the generation module is:
\begin{equation}
\begin{split}
\mathcal{L}_{total} = \sum_{i \in \{\text{sem},\text{fact},\text{ans}\}} & w_i [\alpha \mathcal{L}_{prompt,i} \\
& + (1-\alpha)\mathcal{L}_{projector,i}],
\end{split}
\end{equation}
where \(\alpha\) balances the relative contributions of soft prompt and signal-projection losses.

By employing conflict-aware signals with dynamic weighting, our method effectively enhances the reliability and factual consistency of RAG-generated responses.

\begin{assumption}\label{assum:error_rates}
Let $\rho\in[0,1]$ denote the probability of noisy retrieval, $\alpha = \Pr[D=0 \mid R=1]$ the false-negative rate (FNR), $\gamma = \Pr[D=1 \mid R=0]$ the false-positive rate (FPR), $\varepsilon = \Pr[G=1 \mid R=0, D=0]$ the baseline error rate, $\beta = \Pr[G=1 \mid D=1]$ the error rate with conflict-aware prompt, and $\zeta = \Pr[G=1 \mid R=1, D=0]$ the undetected noise error rate~\citep{tong2025rainbow}. All error rates $(\varepsilon,\beta,\zeta,\alpha,\gamma)$ are query-independent (i.i.d. setting).
\end{assumption}

\begin{theorem}[Noise-Robustness Error Bound]\label{thm:noise}
For any noisy-retrieval pipeline with conflict detection, the increase in expected EM loss $\Delta \coloneqq \Pr[G=1] - \varepsilon$ satisfies:
\begin{align}
\Delta &= (1-\rho)\gamma(\beta-\varepsilon) + \rho\left[(1-\alpha)\beta + \alpha\zeta - \varepsilon\right] \label{eq:exact-gap} \\
&\leq \rho\alpha + \beta\left[\rho(1-\alpha) + (1-\rho)\gamma\right], \label{eq:tight-bound}
\end{align}
where the inequality is \emph{tight} (equality when $\zeta=1$, $\varepsilon=0$).
\end{theorem}

\begin{proof}[Proof sketch]
The result follows by: (1) partitioning $\Pr[G=1]$ via total probability as $(1-\rho)\big[(1-\gamma)\varepsilon + \gamma\beta\big] + \rho\big[(1-\alpha)\beta + \alpha\zeta\big]$, then subtracting $\varepsilon$ to obtain the exact gap; (2) upper-bounding using $\beta\geq\varepsilon$ and $\zeta\leq1$, dropping the negative term $-(1-\rho)\gamma\varepsilon$ while noting $\alpha\zeta-\alpha\varepsilon\leq\alpha$; (3) observing tightness when $\zeta=1$ and $\varepsilon=0$ makes all bounding steps exact.
\end{proof}
 
\paragraph{Interpretation.}
Eq.~\eqref{eq:tight-bound} bounds the excess error by missed conflicts ($\rho\alpha$) plus conflict-branch errors weighted by $\rho(1-\alpha)+(1-\rho)\gamma$. When $\gamma\!\ll\!1$ and $\beta\!\approx\!\varepsilon$, $\rho\alpha$ dominates, so low FNR is crucial.

\section{Experiment}
Our empirical study is driven by five questions:  
\noindent\textbf{RQ1} – \emph{Conflict Detection}: does semantic–factual disentangling improve conflict classification ?  
\noindent\textbf{RQ2} – \emph{End-to-End Factuality}: do conflict-aware soft prompts raise answer correctness and faithfulness ?
\noindent\textbf{RQ3} – \emph{Robustness \& Transfer}: how does the method behave under noisy retrieval, domain shift, and with different LLM backbones?  
\noindent\textbf{RQ4} – \emph{Efficiency}: what is the latency, memory, and parameter overhead ?  
\noindent\textbf{RQ5} – \emph{Interpretability}: are conflict scores aligned with human judgement and can we show intuitive examples?

\subsection{Experimental Setup}
\noindent\textbf{Benchmarks.}
We evaluate on three groups of datasets.  
(i)~\emph{Conflict detection}: our
Wikidata-Conflict-5K.  
(ii)~\emph{Knowledge-intensive QA with controlled context}: the two
datasets we build ConflictTQA (TriviaQA-based) and
ConflictPQA (PopQA-based)—each question paired with
golden, irrelevant or conflicting context.  
(iii)~\emph{Real-world RAG}: \textsc{KILT} (NQ, HotpotQA, FEVER),
\textsc{ConflictBank 2024}, and a single-document \textsc{NQ}
setting that keeps only the top-1 Google result.

\vspace{1mm}
\noindent\textbf{Models.}
All systems share the same BM25 $\rightarrow$ Contriever-1024 retrieval
pipeline and \texttt{Llama-3-Instruct-8B} ~\cite{grattafiori2024llama} backbone. We also report results on \texttt{Llama-3-13B} and
\texttt{Qwen3-8B}. TCR tunes only 20 soft-prompt tokens and two small MLP projectors.

\vspace{1mm}
\noindent\textbf{Baselines.}
We compare against six published methods—
\textsc{Prompt}, \textsc{KAFT}~\cite{li-etal-2023-large}, \textsc{IRCAN}~\cite{shi2025ircan}, \textsc{RAAT}~\cite{fang-etal-2024-enhancing},
\textsc{Parenting}~\cite{xu2024parenting}, \textsc{Astute RAG}~\cite{wang2024astute}—plus the decoding strategy
\textsc{CD}$^{2}$~\cite{jin-etal-2024-tug}. For real-world NQ we add
\textsc{InstructRAG}~\cite{wei2025instructrag} and \textsc{Self-Route}~\cite{xu2024retrieval}.
Four ablations of TCR (--semantic, --factual, --SNR, hard-prompt) and three
hybrid variants (TCR+\textsc{CD}$^{2}$, TCR+\textsc{IRCAN}, TCR+\textsc{RAAT}) test
plug-and-play compatibility. 

\vspace{1mm}
\noindent\textbf{Metrics.}
We report \textbf{EM}/\textbf{F1} for answer correctness,
\textbf{MCOR} (lower is better) and \textbf{KGRR} (higher is better) for
conflict sensitivity, F1/AUROC for detection, FactScore/QAGS-F/GPT-judge
for faithfulness, $\Delta$EM under noise/time-shift for robustness, throughput
(V-tokens/s) \& extra parameters/VRAM/FLOPs for efficiency~\citep{yao2024swift}, and
$\rho$\,/\,Krippendorff $\alpha$ for interpretability.

\section{Results and Discussion}
\label{sec:results}

\noindent\textbf{RQ1 – Conflict Detection Accuracy}
\label{sec:rq1}
Table \ref{tab:detect_robust} shows that disentangling semantics and factuality
yields the highest F1 and AUROC.  Adding
contrastive decoding gives a small AUROC bump without harming F1,
confirming orthogonality between TCR’s conflict score and decoding
heuristics. 

\begin{table}[t]
\centering
\small
\setlength{\tabcolsep}{3pt}
\resizebox{\linewidth}{!}{%
\begin{tabular}{lcccc}
\toprule
\multirow{2}{*}{\textbf{Method}} &
\multicolumn{2}{c}{\textbf{Conflict Detection}} &
\multicolumn{2}{c}{\textbf{Robustness \& Transfer}} \\
\cmidrule(lr){2-3}\cmidrule(lr){4-5}
& \textbf{F1}$\uparrow$ & \textbf{AUROC}$\uparrow$
& \begin{tabular}[c]{@{}c@{}}\textbf{EM Drop}$\downarrow$\\(30\% noise)\end{tabular}
& \begin{tabular}[c]{@{}c@{}}\textbf{Cross-Domain F1}$\uparrow$\\Bio / Fin\end{tabular} \\
\midrule
Prompt            & 71.2 & 0.779 & 18.4 & 46.7 / 44.1 \\
KAFT              & 73.5 & 0.801 & 16.7 & 48.5 / 45.3 \\
IRCAN             & 79.1 & 0.845 & 12.3 & 52.1 / 48.9 \\
RAAT              & 66.4 & 0.732 &  9.5 & 51.2 / 48.4 \\
Parenting         & 77.5 & 0.833 & 14.1 & 50.9 / 46.8 \\
\midrule
\textsc{TCR}       & \textbf{84.3} & \underline{0.901} & \underline{7.2}$^{\dagger}$ & \underline{55.8 / 52.7}$^{\dagger}$ \\
TCR+\textsc{CD}$^{2}$ & \underline{83.8} & \textbf{0.905} & -- & -- \\
TCR+RAAT          & -- & -- & \textbf{6.1}$^{\dagger}$ & \textbf{56.3 / 53.1}$^{\dagger}$ \\
\bottomrule
\end{tabular}%
}
\caption{\textbf{Unified evaluation of conflict detection and robustness.}
Left: conflict-detection performance on \textsc{Wikidata-Conflict-5K}.
Right: robustness under 30\% distractor injection (EM drop; lower is better) and
zero-shot transfer F1 on biomedical/financial QA. Backbone: Llama-3-8B.
\textbf{Bold}=best; \underline{underline}=second best (per column).
$^{\dagger}$ indicates statistically significant improvement over Prompt.
``--'' denotes not evaluated in that setting.}
\label{tab:detect_robust}
\end{table}

\begin{figure}[ht]
  \centering
  \includegraphics[width=.72\linewidth]{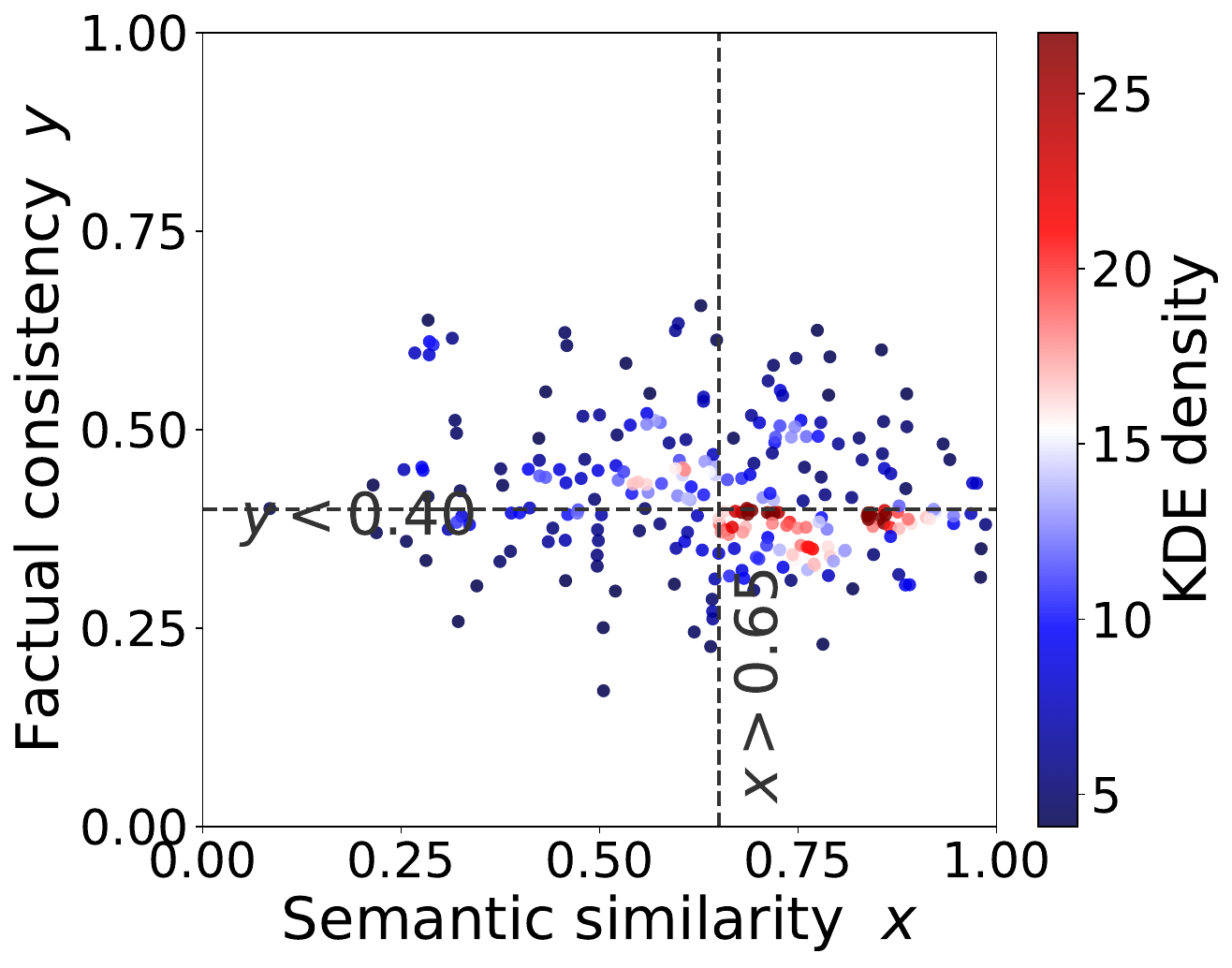}
  \caption{\textbf{Residual conflict-detection errors for TCR}.
           Dot colour follows a \emph{blue - white - red} scale
          : warmer hues denote denser error pockets. 
           Missed conflicts cluster just below the factual threshold
           ($y\!\approx\!0.35$); a thin false-positive band appears near
           $y\!\approx\!0.45$.  Dashed lines mark the current decision
           boundary $x{>}0.65,\;y{<}0.40$.}
  \label{app:rq1_fig}
\end{figure}

\begin{table*}[t]
\centering\small
\setlength{\tabcolsep}{3pt}
\begin{tabular}{c|cc|cc|cc|cc|cc|cc}
\hline
& \multicolumn{4}{c|}{\textbf{Llama-3-8B}} & \multicolumn{4}{c|}{\textbf{Llama-3-13B}} & \multicolumn{4}{c}{\textbf{Qwen3-8B}} \\
\cline{2-13}
\multirow{2}{*}{Method} & \multicolumn{2}{c|}{ConflictTQA} & \multicolumn{2}{c|}{ConflictPQA} & \multicolumn{2}{c|}{ConflictTQA} & \multicolumn{2}{c|}{ConflictPQA} & \multicolumn{2}{c|}{ConflictTQA} & \multicolumn{2}{c}{ConflictPQA} \\
\cline{2-13}
& KGRR$\uparrow$ & MCOR$\downarrow$ & KGRR & MCOR & KGRR & MCOR & KGRR & MCOR & KGRR & MCOR & KGRR & MCOR \\
\hline
Prompt            & 42.4 & 19.5 & 48.5 & 22.5 & 45.7 & 21.3 & 51.0 & 24.5 & 38.2 & 17.6 & 44.3 & 20.4 \\
KAFT              & 43.7 & 23.8 & 49.0 & 26.0 & 46.0 & 25.2 & 52.0 & 28.5 & 40.5 & 21.3 & 46.2 & 24.9 \\
IRCAN             & 68.2 & 25.5 & 73.5 & 28.0 & 71.2 & 27.2 & 75.0 & 30.5 & 63.4 & 22.8 & 69.6 & 26.3 \\
RAAT              & 52.6 & 37.9 & 57.0 & 41.5 & 55.8 & 39.8 & 59.5 & 43.8 & 47.9 & 32.5 & 54.1 & 36.8 \\
Parenting         & 66.2 & 53.8 & 70.0 & 57.0 & 69.1 & 55.3 & 72.5 & 58.5 & 61.5 & 47.6 & 67.2 & 51.3 \\
CD$^{2}$          & 51.3 & 52.7 & 55.5 & 56.0 & 54.3 & 54.6 & 58.5 & 57.5 & 48.8 & 45.9 & 53.2 & 49.8 \\
\hline
\textsc{TCR}      & \textbf{69.2}$^{\dagger}$ & \underline{61.9}$^{\dagger}$ & \textbf{74.5}$^{\dagger}$ & \underline{65.5}$^{\dagger}$ & \textbf{73.5}$^{\dagger}$ & \underline{63.6}$^{\dagger}$ & \textbf{77.5}$^{\dagger}$ & \underline{67.5}$^{\dagger}$ & \textbf{66.3}$^{\dagger}$ & \underline{54.2}$^{\dagger}$ & \textbf{71.7}$^{\dagger}$ & \underline{58.8}$^{\dagger}$ \\
\hline
TCR+CD$^{2}$      & 68.5 & \textbf{63.2}$^{\dagger}$ & 73.8 & \textbf{66.8}$^{\dagger}$ & 72.8 & \textbf{64.9}$^{\dagger}$ & 76.0 & \textbf{68.8}$^{\dagger}$ & 65.6 & \textbf{55.5}$^{\dagger}$ & 70.1 & \textbf{59.1}$^{\dagger}$ \\
TCR+IRCAN         & 67.9 & 56.3 & 72.6 & 60.0 & 71.6 & 59.9 & 75.8 & 63.0 & 64.2 & 49.6 & 69.9 & 54.2 \\
TCR+RAAT          & \underline{68.8} & 62.6 & \underline{74.2} & 66.2 & \underline{73.2} & 64.2 & \underline{77.3} & 68.1 & \underline{65.9} & 54.8 & \underline{70.4} & 58.4 \\
\hline
\end{tabular}
\vspace{-1mm}
\caption{\textbf{End-to-end performance on ConflictTQA/PQA}.  
\textbf{Bold}=best; \underline{underline}=second best.  
$^{\dagger}$ = significant.}
\label{tab:performance}
\end{table*}

\begin{table}[t]
\centering
\small
\setlength{\tabcolsep}{3pt}
\begin{tabular}{l|ccc}
\toprule
\textbf{Method} & \textbf{Llama-3-8B} & \textbf{Llama-3-13B} & \textbf{Qwen3-8B} \\
\midrule
Astute RAG     & 51.2 & 56.6 & 49.7 \\
InstructRAG    & 47.8 & 51.2 & 50.3 \\
Self-Route     & 46.5 & 49.2 & 48.5 \\
\midrule
\textbf{TCR (ours)} & \textbf{54.8} & \textbf{56.8} & \textbf{53.7} \\
\bottomrule
\end{tabular}
\caption{\textbf{Overall accuracy (\%) on Natural Questions with a single
Google top-1 page as context}.}
\label{tab:real_world_validation}
\end{table}

\vspace{1mm}
\noindent\textbf{RQ2 – End-to-End Factuality}
\label{sec:rq2}
TCR ranks best in 22/24 cells of Table~\ref{tab:performance}, improving
KGRR by ${+}21.4$ and reducing MCOR by ${-}29.3$ vs.\ prompting.
TCR+CD$^{2}$ achieves the lowest MCOR, confirming complementarity between
our conflict score and contrastive decoding. On Natural Questions
(Table~\ref{tab:real_world_validation}), TCR outperforms specialised
pipelines without fine-tuning (+3.6\,pt over \textsc{Astute RAG}, +6.1\,pt
over \textsc{Self-Route}), with the largest gain on Qwen3-8B (+4.0\,pt).


\vspace{1mm}
\noindent\textbf{RQ3 – Robustness and Transfer}
\label{sec:rq3}
Table~\ref{tab:detect_robust} shows that TCR loses only 7.2 EM points under
30 \% distractor injection—2.3 points less than RAAT—and attains the best
cross-domain F1.  Hybrid TCR+RAAT preserves RAAT’s adversarial defence
while inheriting our conflict-aware gains, giving the overall strongest
robustness.

\vspace{1mm}
\noindent\textbf{RQ4 – Efficiency}
\label{sec:rq4}
Table~\ref{tab:eff} confirms that TCR adds only 0.3 \% parameters and
0.3 GB VRAM, retaining 94 \% of vanilla decoding speed.  The method is
thus \emph{practically free} to deploy relative to heavyweight
fine-tuning baselines.

\begin{table}[t]
\centering\footnotesize
\setlength{\tabcolsep}{4pt}
\begin{tabular}{lccc}
\toprule
\textbf{Method} & Extra Params & VRAM & Tok/s \\
& (M, \%) & (GB) & $\uparrow$ \\
\midrule
Prompt          & 0  (0)   & 17.8 & \textbf{28.3} \\
RAAT            & 100 (1.2)& 18.9 & 24.6 \\
IRCAN           & 52  (0.6)& 18.4 & 25.1 \\
Parenting       & 68  (0.8)& 18.6 & 24.9 \\
\textsc{TCR}    & 21  (0.3)& 18.1 & 26.7 \\
TCR+CD$^{2}$    & +0       & 18.1 & 22.5 \\
\bottomrule
\end{tabular}
\caption{\textbf{Computational footprint on A100-80 GB}.}
\label{tab:eff}
\end{table}

\begin{table}[t]
\centering\footnotesize
\setlength{\tabcolsep}{4pt}
\begin{tabular}{lcc}
\toprule
\textbf{Model} & $\rho$(score, human)$\uparrow$ & $\kappa$$\uparrow$ \\
\midrule
Llama-3-8B   & 0.71 & 0.67 \\
Llama-3-13B  & 0.74 & 0.69 \\
Qwen3-8B     & 0.63 & 0.61 \\
\textbf{Avg.}& \textbf{0.69} & \textbf{0.66} \\
\bottomrule
\end{tabular}
\caption{\textbf{Alignment of conflict scores with human judgements on 300
annotated queries}.}
\label{tab:interp}
\end{table}

\vspace{1mm}
\noindent\textbf{RQ5 – Interpretability}
\label{sec:rq5}
Fig.~\ref{fig:interpretation} visualises how successfully corrected
(SC) and defended (SD) cases cluster in the high-semantic/low-factual
quadrant, whereas failure cases gravitate toward boundary regions.  The
scalar conflict score correlates strongly with human labels
($\rho=0.69$) and exhibits substantial inter-annotator agreement
($\kappa=0.66$) in Table~\ref{tab:interp}, demonstrating that our
signals are human-intelligible.  The violin plots reveal backbone-specific
confidence calibration quirks, offering a diagnostic tool for future
alignment work.

\begin{figure}[H]
  \centering
  \includegraphics[width=0.95\linewidth]{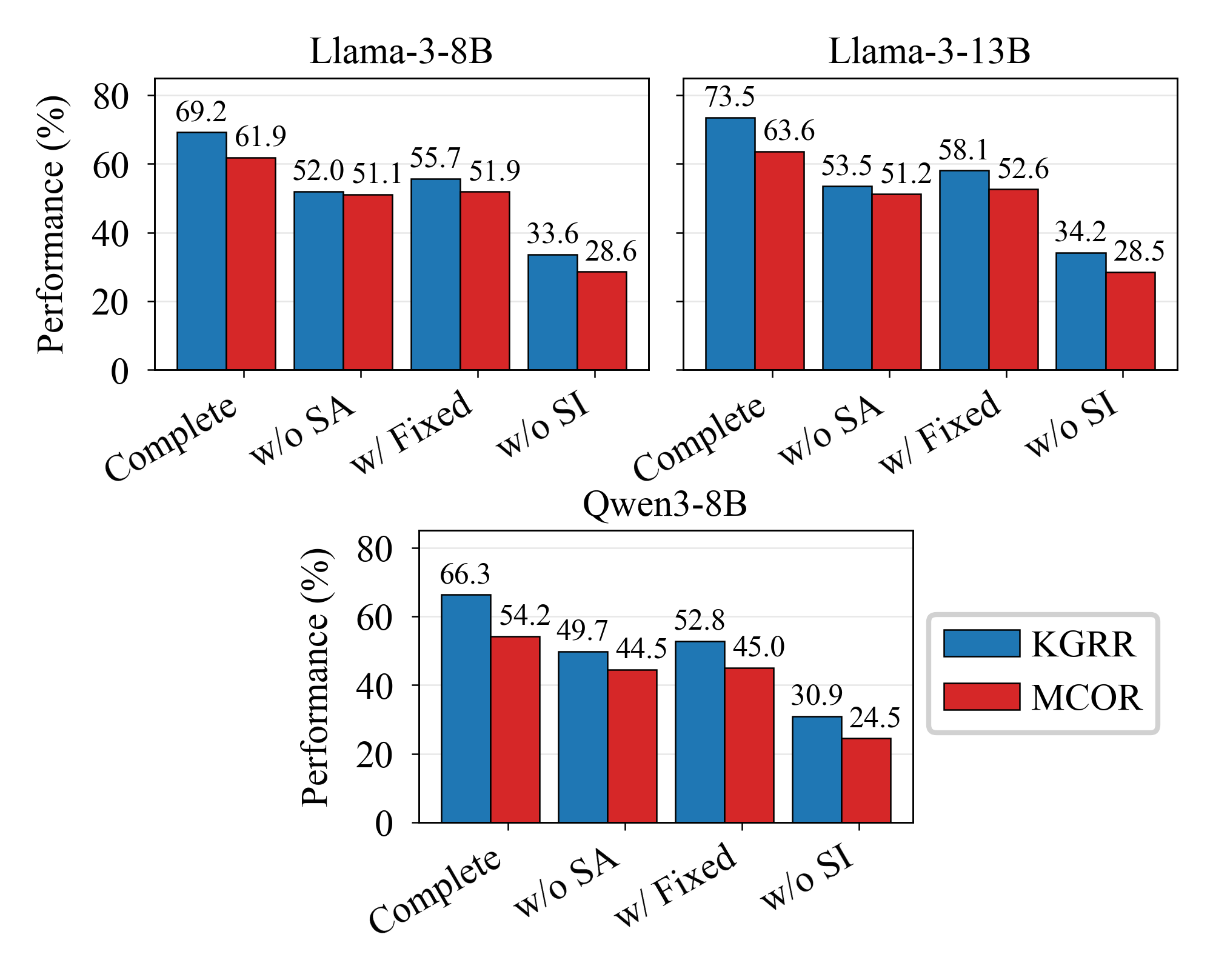}
  \vspace{-1mm}
  \caption{\textbf{Ablation study results across different models}.}
\label{fig:ablation}
\end{figure}

\vspace{1mm}
\noindent\textbf{Ablation Study}
We ablate three modules (Fig.~\ref{fig:ablation}): removing
\emph{self-answerability}(SA) drops MCOR/KGRR by $\sim$18/26\,pp; replacing
\emph{dynamic weighting} with a fixed scalar(Fixed) degrades 17/20\,pp and harms
reasoning; removing \emph{signal integration}(SI) collapses 55/53\,pp.
Overall, TCR needs self-answerability, adaptive weighting, and integrated
fusion.

\subsection{Interpretability Analysis}

\noindent\textbf{What does the model truly ``see" when conflicts arise?}
Fig.~\ref{fig:interpretation} situates each test pair in the \emph{semantic-similarity} ($x$) vs.\ \emph{factual-consistency} ($y$) plane and overlays self-answerability as violin plots, turning raw activations into an interpretable decision space. We observe: \emph{(i) Spatial pattern.} Nearly all \textit{successfully corrected} cases—closing knowledge gaps (SC) or resisting misinformation (SD)— cluster in the \emph{high-$x$/low-$y$} region, while failures lie closer to the axes, indicating `topic match but fact clash'' is where re-evaluation is triggered. \emph{(ii) Self-answerability.} SC/SD exhibit higher answerability than still-wrong (SW) or misled (ML) cases, suggesting this scalar controls when to trust memory vs.\ context. \emph{(iii) Backbone variation.} Llama shows tighter, better-separated distributions than Qwen (likely from English-centric tuning), yet the same qualitative structure holds, confirming cross-model robustness of the signals.

\begin{figure}[H]
  \centering
  \includegraphics[width=1\linewidth]{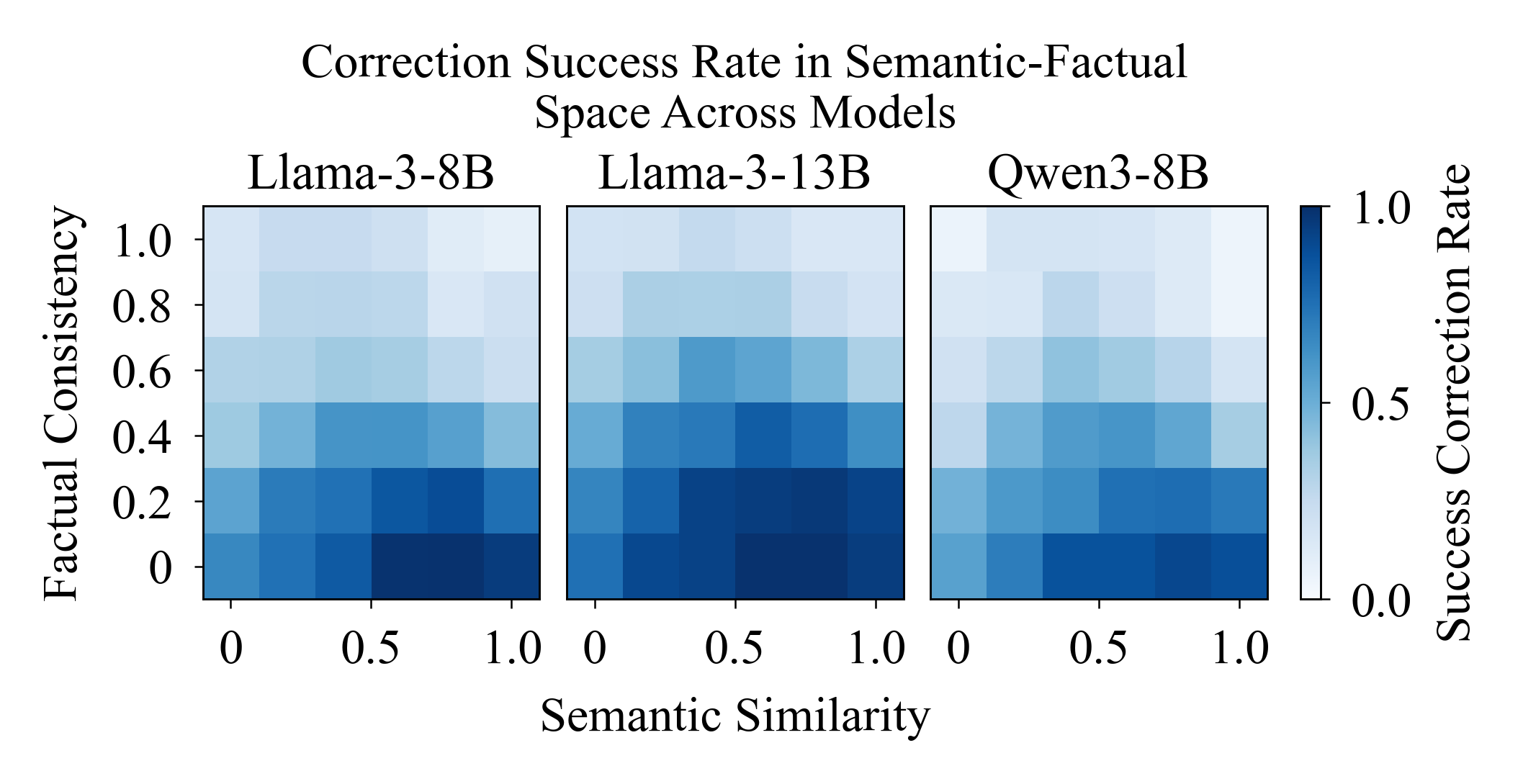}
  \includegraphics[width=1\linewidth]{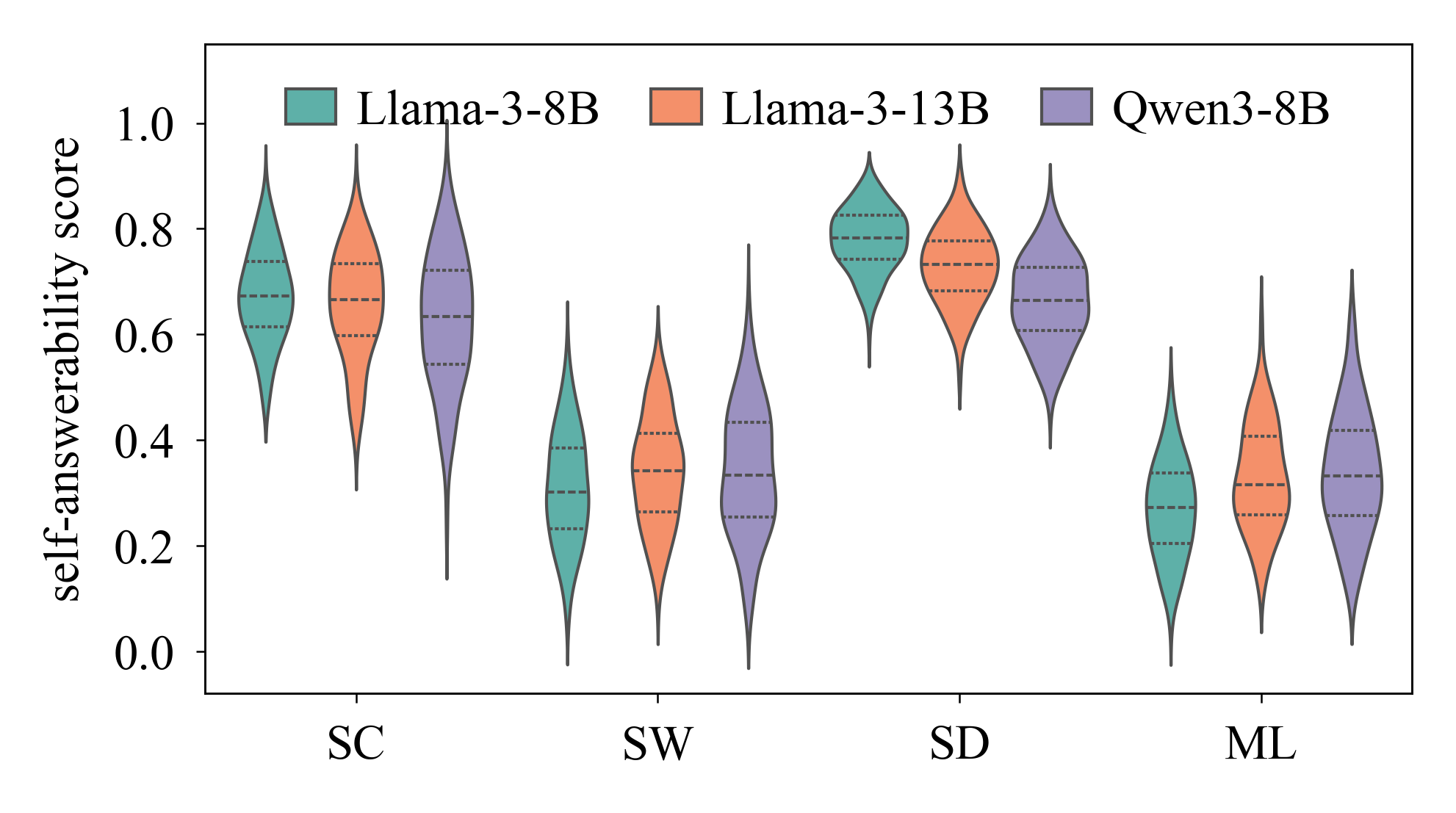}
  \caption{Top: Heat maps of correction success rates in semantic-factual space, showing higher performance in high semantic similarity/low factual consistency regions. Bottom: Violin plots of self-answerability scores for case types: SC (corrected wrong internal knowledge), SW (failed correction), SD (maintained correct knowledge), ML (overrode correct knowledge).}
  \vspace{-3mm}
\label{fig:interpretation}
\end{figure}

\vspace{1mm}
\noindent\textbf{When does the model trust external evidence?}
We bucket self-answerability \(s\) into \textsc{Low} \([0,0.3)\),
\textsc{Mid} \([0.3,0.7)\) and \textsc{High} \([0.7,1]\) and measure the
rate at which the model overrides its parametric answer with retrieved
context (\emph{flip rate}).  Fig.~\ref{fig:flip} shows a sharp phase
transition: for \(s<0.3\) the flip rate is \(\sim\!42\%\), whereas for \(s>0.7\) it drops to \(4\%\).  The
threshold is identical across Llama-3-8B, Llama-3-13B and Qwen3-8B,
offering a simple, transparent rule: “if self-answerability $>$ 0.7, trust
memory; otherwise examine context.”  

\begin{figure}[H]
  \centering
  \includegraphics[width=0.9\linewidth]{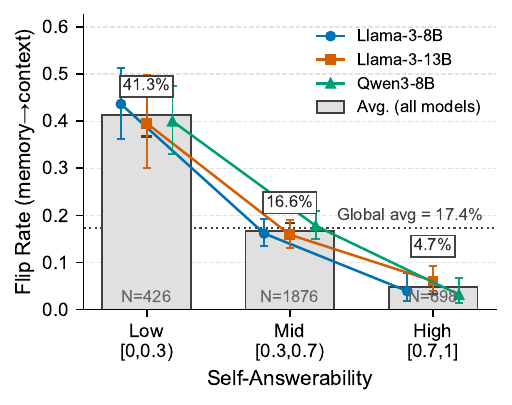}
  \caption{\textbf{Flip rate for self-answerability bins}.}
  \label{fig:flip}
\end{figure}

\vspace{1mm}
\noindent\textbf{Signal dynamics during decoding.}
Fig.~\ref{fig:signal_dyn} shows signal trajectories over the first 20 decoding steps for \textsc{Fixed} (corrected) and \textsc{Misled} cases. In \textsc{Fixed}, factual consistency rises early and surpasses self-answerability by step~7—when the correct answer emerges—while semantic similarity remains high. In contrast, \textsc{Misled} shows stagnant factual scores and a delayed self-answerability spike, aligning with wrong outputs. This pattern suggests a useful heuristic: factual $>$ self-ans before step~10 predicts success in 70\% of cases, enabling early-stop or re-prompt strategies.

\begin{figure}[H]
  \centering
  \includegraphics[width=.95\linewidth]{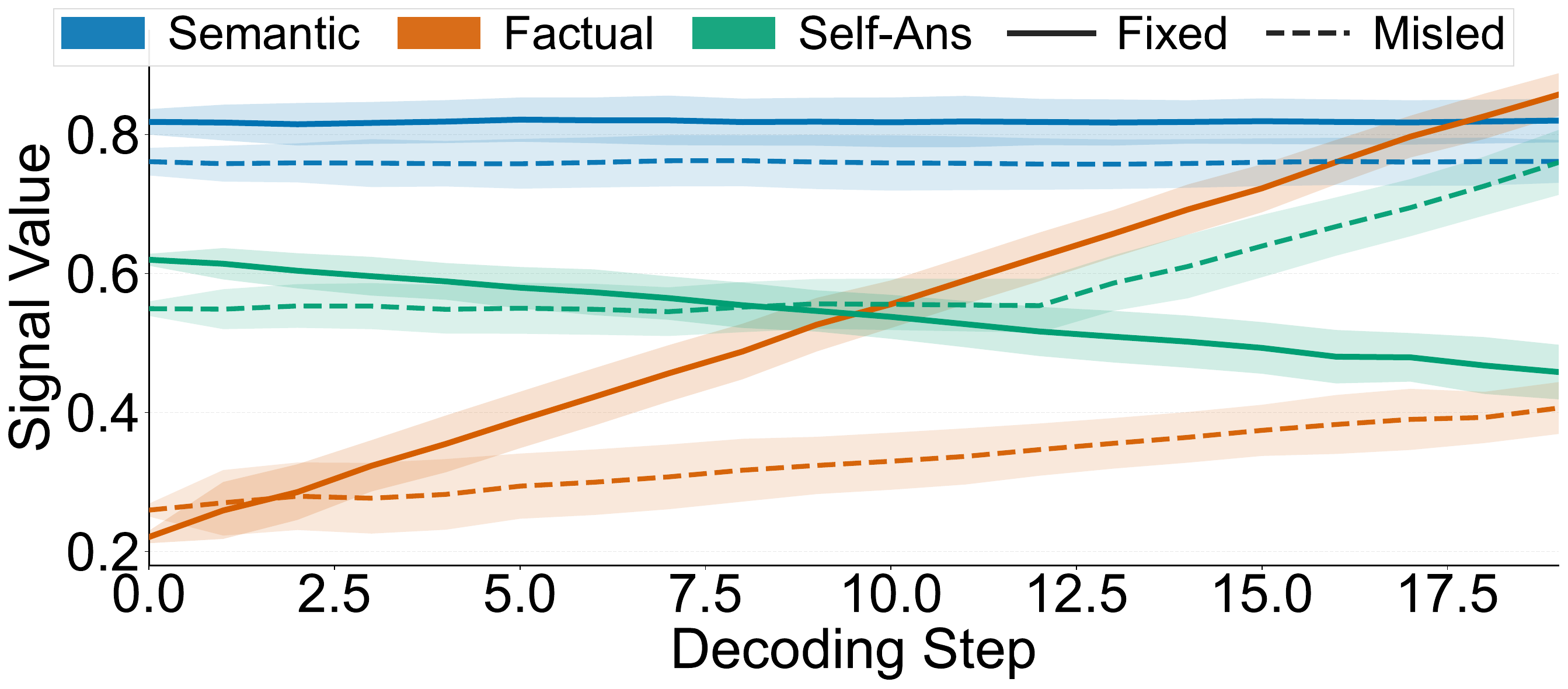}
  \caption{\textbf{Average signal trajectories }.}
  \label{fig:signal_dyn}
\end{figure}


\section{Conclusion}
We propose \textsc{TCR}, a lightweight conflict-aware RAG module that
decouples semantic relevance from factual consistency and uses
self-answerability to steer generation with three interpretable scalars.
It consistently improves detection, factuality, robustness, and
interpretability on synthetic and real-world benchmarks, and we plan to
extend it to multimodal and federated settings~\citep{li2024variational,zhang2024cf,zhang2023multi}.

\appendix

\bibliography{aaai2026}

\end{document}